\documentclass[11pt, a4paper]{article}  
\pdfoutput=1 
\usepackage{cite}

\usepackage[pdftex]{graphicx}
\usepackage[font={small,it},center]{caption} 
\usepackage{subfig}
\usepackage{booktabs}
\usepackage{authblk}

\usepackage{tikz}
\usetikzlibrary{positioning,shapes,arrows,calc,intersections,patterns}
\usepackage{pgfplots}
	\pgfplotsset{compat=1.5} 

\usepackage{fixltx2e} 

\usepackage{color}
\usepackage{courier} 
\usepackage{textcomp} 

\usepackage[cmex10]{amsmath}
\usepackage{amsfonts} 
\usepackage{amssymb} 

\usepackage{amsthm} 

	\newtheorem{theorem}{Theorem}
	
	\newtheorem{definition}{Definition}

	\renewcommand{\vec}[1]{\boldsymbol{#1}} 

	\newcommand{\dist}{\operatorname{Dist}}	

\usepackage{algorithmic} 

\usepackage{listings} 

	\definecolor{lstBlue}{RGB}{11,72,107}
	\definecolor{lstRed}{RGB}{254,67,101}
	\definecolor{lstGreen}{RGB}{81,149,72}
	\definecolor{lstPink}{RGB}{252,157,154}

	\lstdefinelanguage{Prism}{ 
        keywords= {bool, C, ceil, const, ctmc, double, dtmc, endinit, endmodule, endrewards, endsystem, F, false, floor, formula, G, global, I, init, int, label, max, mdp, min, module, nondeterministic, P, Pmin, Pmax, prob, probabilistic, R, rate, rewards, Rmin, Rmax, S, stochastic, system, true, U, X}, 
        comment=[l] {//}, morecomment=[s]{/*}{*/}, 
        captionpos=b, 
	}
	
	\lstdefinelanguage{sEnglish}{
		keywords={., +, -, ~, ^, If, while, then, end, and, or, not},
		keywords=[2]{INITIAL, BELIEFS, AND, GOALS, ACTIONS, PERCEPTION, PROCESS, REASONING, EXECUTABLE, PLANS},
    	morecomment=[l]{//},
    	moredelim=[s][\color{black}]{\{}{\}},
		moredelim=[s][\color{black}]{^[}{]},
		moredelim=[s][\color{black}]{[}{]},
    	morestring=[b]",
    	morestring=[b]',
    }
    \lstdefinelanguage{mol}{
    	morecomment=[l][\color{black}]{>},
    	moredelim=[s][\color{black}]{@}{:},
    }

\usepackage{url}
\usepackage{hyperref} 
    \hypersetup{
        colorlinks,
        citecolor=black,
        filecolor=black,
        linkcolor=black,
        urlcolor=black
    }

\usepackage{array} 
\usepackage{supertabular}

\usepackage{glossaries}
	\glsdisablehyper

	\newacronym{ail}{AIL}{Agent Infrastructure Layer}
	\newacronym{ajpf}{AJPF}{Agent Java PathFinder}
    \newacronym{agv}{AGV}{Autonomous Ground Vehicle}
    \newacronym{aop}{AOP}{Agent Oriented Programming}
    \newacronym{asv}{ASV}{Autonomous Surface Vehicle}
    \newacronym{auv}{AUV}{Autonomous Underwater Vehicle}
    
    \newacronym{bdi}{BDI}{Belief-Desire-Intention}
    \newacronym{bdd}{BDD}{Binary Decision Diagram}    
    
    \newacronym{cat}{CAT}{Cognitive Agent Toolbox}
    \newacronym{cdt}{CDT}{Constrained Delaunay Triangulation}
    \newacronym{cog}{COG}{Centre Of Gravity}
    \newacronym{ctl}{CTL}{Computation Tree Logic}
    
    \newacronym{dof}{DOF}{Degrees Of Freedom}
    \newacronym{dtmc}{DTMC}{Discrete-Time Markov Chain}
    
    \newacronym{ekf}{EKF}{Extended Kalman Filter}
    
    \newacronym{gps}{GPS}{Global Positioning System}
    \newacronym{gui}{GUI}{Graphical User Interface}
    
    \newacronym{hs}{HS}{Hybrid System}
    
    \newacronym{imu}{IMU}{Inertial Measurement Unit}
    \newacronym{ivp}{IvP}{Interval Programming}
    
    \newacronym{jpf}{JPF}{Java PathFinder}
    
    \newacronym{ltl}{LTL}{Linear Temporal Logic}  
    \newacronym{lisa}{LISA}{Limited Instruction Set Agent}
    \newacronym{lct}{LCT}{Local Clearance Triangulation}
    
    \newacronym{mcmas}{MCMAS}{Model Checker for Multi-Agent Systems}
    \newacronym{mdp}{MDP}{Markov Decision Process}
    \newacronym{mol}{MOL}{Machine Ontology Language}
    \newacronym{moos}{MOOS}{Mission Oriented Operating Suite}
    \newacronym{mtbdd}{MTBDD}{Multi-Terminal Binary Decision Diagram}
    
    \newacronym{nasa}{NASA}{National Aeronautics and Space Administration}
    \newacronym{nlp}{NLP}{Natural Language Programming}
    
    \newacronym{oop}{OOP}{Object Oriented Programming}
    
    \newacronym{pact}{PACT}{Pilot Authority and Control of Tasks}
    \newacronym{pctl}{PCTL}{Probabilistic Computation Tree Logic}
    \newacronym{prs}{PRS}{Procedural Reasoning System}
    \newacronym{pta}{PTA}{Probabilistic Timed Automata}
    
    \newacronym{rrt}{RRT}{Rapidly-exploring Random Tree}
    
    \newacronym{sis}{SIS}{Sequential Importance Sampling}
    \newacronym{slam}{SLAM}{Simultaneous Localization And Mapping}
    \newacronym{smcp}{SMCP}{Standard Marine Communication Phrases}
    
    \newacronym{ukf}{UKF}{Unscented Kalman Filter}
    \newacronym{usv}{USV}{Unmanned Surface Vehicle}
    \newacronym{uav}{UAV}{Unmanned Aerial Vehicle}
    \newacronym{uuv}{UUV}{Unmanned Underwater Vehicle}
    
    \newacronym{vr}{VR}{Virtual Reality}

\begin{document}

\title{
A stochastically verifiable autonomous control architecture with reasoning
}

\author{Paolo Izzo, Hongyang Qu and Sandor M. Veres}
\affil{Department of Automatic Control and Systems Engineering\\ University of Sheffield,
Sheffield 
S1 3JD, UK \\\{pizzo1, h.qu, s.veres\}@sheffield.ac.uk}

\maketitle
\thispagestyle{empty}
\pagestyle{empty}

\begin{abstract}

A new agent architecture called Limited Instruction Set Agent (LISA) is introduced for autonomous control. The new architecture is based on previous implementations of AgentSpeak and it is structurally simpler than its predecessors with the aim of facilitating design-time and run-time verification methods. The process of abstracting the LISA system to two different types of discrete probabilistic models (DTMC and MDP) is investigated and illustrated. The LISA system provides a tool for complete modelling of the agent and the environment for probabilistic verification. The agent program can be automatically compiled into a \gls{dtmc} or a \gls{mdp} model for verification with \textsc{Prism}. The automatically generated \textsc{Prism} model can be used for both design-time and run-time verification. The run-time verification is investigated and illustrated in the LISA system as an internal modelling mechanism for prediction of future outcomes.

\end{abstract}

\section{Introduction}

Autonomous control is an area within control sciences that emerged by upgrading  classical \emph{feedback control} by decision making on what control references to use. The purpose of feedback control is to regulate a system in order to make it follow a predefined reference input. Autonomous controllers are designed to make decisions \emph{what} reference signal to use and, more generally, what goals to achieve and \emph{how} to achieve them. They do so by generating and executing plans of action that work toward goals \cite{astrom1992,astrom2010}. Autonomous controllers aim to introduce a certain level of ``intelligence'' in control systems, that is the ability of a system to act appropriately in an uncertain environment\cite{meystel2002}.

A first attempt towards autonomous decision-making software was initially made by using \gls{oop}.  
However the passive nature of objects in \gls{oop}, led to the development of active objects called ``agents'' \cite{veres2011}, which implement decision-making processes. A formal description of autonomous agents can be found in \cite{veres2011,wooldridge2002,wooldridge1995}. 

One of the most widely used ``anthropomorphic'' approaches to the implementation of autonomous agents is the \gls{bdi} architecture \cite{bordini2007,veres2011}. \gls{bdi} agent architectures are characterised by three large sets of atomic predicates: \emph{Beliefs}, \emph{Desires} and \emph{Intentions}. The most known implementations of the \gls{bdi} architecture are the \gls{prs} \cite{georgeff1986,georgeff1987} and \emph{AgentSpeak} \cite{rao1996}. AgentSpeak fully embraces the philosophy of \gls{aop} \cite{shoham1993}, and it offers a customisable Java based interpreter.

Autonomous agents have a considerable potential for implementation in all sorts of different applications. However their introduction in real-world scenarios brings along safety concerns, creating the need for \emph{model checking} \cite{clarke1999}. 
An early attempt to \gls{bdi} agent verification can be found in \cite{bordini2003,bordini2006}, where the authors present a translation software from \emph{AgentSpeak} to either Promela or Java, and then use the associated model checkers Spin \cite{holzmann1991,holzmann1997} and \gls{jpf} \cite{visser2003}. 
A subsequent effort towards verifiable agents was made by Dennis \emph{et al.} \cite{dennis2008} with a \gls{bdi} agent programming language called \emph{Gwendolen}, which is implemented in the \gls{ail} \cite{dennis2008b,dennis2008c}, a collection of Java classes intended for use in model checking agent programs, particularly with \gls{jpf}. An evolution of \gls{jpf} is \gls{ajpf} \cite{dennis2012}, specifically designed to verify agent programs.
However \gls{jpf} and \gls{ajpf} introduce a significant bottleneck in the workflow as the internal generation of the program model, which is created by executing all possible paths, is highly computationally expensive. In \cite{hunter2013} it is suggested to alleviate this problem by using \gls{jpf} to generate models of agent programs that can be executed in other model-checkers. This idea is further developed in \cite{dennis2015}, which shows how \gls{ajpf} can be modified to output models in the input languages of Spin or \textsc{Prism} \cite{kwiatkowska2011}, a probabilistic model checker.  
All of the approaches towards agent verification to date, do not provide the user with a complete framework to build and verify a probabilistic model, and mostly they do not perform at a level suitable for real-time applications.


In this paper we introduce a new agent architecture called \gls{lisa}. The architecture of \gls{lisa} is based on the three-layer architecture \cite{gat1998} and the agent program is an evolution of Jason \cite{bordini2007,jasonmanual}.
The aim is to simplify the structure and the execution of the agent program in order to reduce the size of the state-space required to abstract it and ultimately allow for a fast verification process.
The agent program is developed and described with \emph{sEnglish} \cite{lincoln2013,veres2008}, a natural language programming interface. The use of sEnglish provides a way to define both the agent program and the environment model in an intuitive, natural-language document. 
The document will then be automatically translated into \textsc{Prism} source code for verification by probabilistic model checking. This is done by first proving that \gls{lisa} can be abstracted as a \gls{dtmc} or a \gls{mdp}, based on design choices made by the user.
We also propose the use of probabilistic model checking in \textsc{Prism} to improve the non-deterministic decision making capabilities of the agent in a run-time verification process. Using run-time verification the agent is able to look into the consequences of its own choices by running model checking queries through the previously generated \textsc{Prism} model.


\section{Background}
\label{sec:background}

\subsection{Rational Agents}

An agent-based system is characterised by its \emph{architecture}, a description of how the agent reasoning communicates with lower abstraction subsystems and ultimately with the environment.
By analogy to previous definitions \cite{lincoln2013,wooldridge2002,veres2011}, we define the agent reasoning as follows.
\begin{definition}[\textbf{Rational agent}]
\label{def:agent}
A \emph{rational \gls{bdi} agent} is defined as a tuple 
\begin{equation*}
\label{eq:agent}
\mathcal{R}=\{ \mathcal{F},B,B_0,L,A,A_0,\Pi\}
\end{equation*}
where:
\begin{itemize}
\item 
	$\mathcal{F} = \{p_1,p_2,\ldots,p_{n_p}\}$ is the set of all predicates.
\item
	$B \subset \mathcal{F}$ is the total atomic Beliefs set.
\item
	$B_0$ is the Initial Beliefs set.
\item
	$L = \{l_1,l_2,\ldots\,l_{n_l}\}$ is a set of logic-based implication rules on the predicates of $B$.
\item
	$A = \{a_1,a_2,\ldots,a_{n_a}\} \subset \mathcal{F} \setminus B$ is a set of all available actions. Actions can be either \emph{internal}, when they modify the Beliefs set to generate internal events, or \emph{external}, when they are linked to external functions. Beliefs generated by internal actions are also called `mental notes'.
\item
	$A_0$ is the set of Initial Actions.
\item 
	$\Pi = \{\pi_1,\pi_2,\ldots,\pi_{n_\pi}\}$ is the set of executable plans or \emph{plans library}. Each plan $j$ is a sequence $\pi_j(\lambda_j)$, with $\lambda_j \in [0,n_{\lambda_j}]$ being the \emph{plan index}, where $\pi(0)$ is a logic statement called \emph{triggering condition}, and $\pi_j(\lambda_j)$ with $\lambda_j>0$ is an action from $A$.	
\end{itemize}
\end{definition}

During an execution the agent also uses the following dynamic subsets of the sets defined above:

\begin{itemize}
\item
	$B[t]\subset B$ is the \emph{Current Beliefs} set, the set of all beliefs available at time $t$. Beliefs can be negated with a `\texttildelow' symbol.
\item
	$E[t] \subset B$ is the \emph{Current Events} set, which contains events available at time $t$. An \emph{event} is a belief paired with either a `$+$' or a `$-$' operator to indicate that the belief is either added or removed.
\item
	$D[t] \subset \Pi$ is the \emph{Applicable Plans} or \emph{Desires} set at time $t$, which contains plans triggered by current events.
\item
	$I[t] \subset \Pi$ is the \emph{Intentions} set, which contains plans that the agent is committed to execute.
\end{itemize} 

The \emph{triggering condition} of each plan of the plan library is composed by two parts: a \emph{triggering event} and a \emph{context}, a logic condition to be verified for the plan to apply.
We write $B[t]\vDash c$ when the Current Beliefs set ``satisfies'' the expression $c$, or in other words when the conditions expressed by $c$ are true at time $t$. Note that in all our definitions and throughout the paper, time $t\in\mathbb{N}_{\geq 1}$ refers to the integer count of reasoning cycles. 

Although different \gls{aop} languages implement the agent in different ways, generally speaking an agent program is iterative. Each iteration is called \emph{reasoning cycle}. The reasoning cycle of the \gls{lisa} system is explained in Sec. \ref{sec:rescycle}.



\subsection{Model checking and verification}
\label{sec:modelcheck}

Probabilistic model checking is an automated verification method that aims to verify the correctness of probabilistic systems, by establishing if a desired property holds in a probabilistic model of the system \cite{kwiatkowska2010}. 
For the purpose of this work we will consider models in particular: \gls{dtmc} and \gls{mdp}. Referring to \cite{forejt2011,kwiatkowska2007,kwiatkowska2010} we give the following definitions:

\begin{definition} [\acrfull{dtmc}]
	A (labelled) \gls{dtmc} is a tuple $\mathcal{D}=\left( S,s_0,\vec{P},L \right)$, where $S$ is a countable set of states, $s_0 \in S$ is the initial state, $\vec{P}: S \times S \rightarrow [0,1]$ is a \emph{Transition Probability Matrix} such that $\sum_{s^\prime \in S} \vec{P}(s,s^\prime)=1$ and $L: S \rightarrow \wp(\mathcal{F})$ is a labelling function that assigns to each state a set of atomic prepositions that are valid in the state. 
\label{def:dtmc}
\end{definition}

\begin{definition}[\acrfull{mdp}]
	A (labelled) \gls{mdp} is a tuple $\mathcal{M}=\left( S,s_0,C,Step,L \right)$, where $S$ is a countable set of states, $s_0 \in S$ is the initial state, $C$ is an alphabet of choices with $C(s^\prime)$ being the set of choices available in any state $s^\prime$, $Step: S\times C \rightarrow \dist(S)$ is a probabilistic transition function with $\dist(S)$ being the set of all probability distributions over $S$ and $L: S \rightarrow \wp(\mathcal{F})$ is a labelling function that assigns to each state a set of atomic prepositions that are valid in the state.
\label{def:mdp}
\end{definition}

Detailed explanation on the techniques used to perform model checking on probabilistic models goes beyond the scope of this paper. However we report here the syntax of the language used to write properties to verify with model checkers, which is called \gls{pctl} \cite{hansson1994}.
\begin{definition}[Syntax of \acrshort{pctl}]
	\begin{equation*}
	\begin{array}{rl}
		\phi ::= & \mathtt{true} \;|\; a \;|\; \phi \wedge \phi \;|\; \neg \phi \;|\; \mathtt{P}_{\bowtie p}[\psi]\\
		\psi ::= & \mathtt{X}\: \phi \;|\; \phi\: \mathtt{U}^{\leq k}\: \phi \\
	\end{array}
	\end{equation*}
	where $a$ is an atomic proposition, $\bowtie \in \{\leq,<,>,\geq\}$ and $p\in [0,1]$.
\end{definition}
We also allow the usual abbreviations such as `$\mathtt{F} \phi$' (equivalent to `$\mathtt{true}\: \mathtt{U}\: \phi$'). A commonly used extension of \gls{pctl} is the addition of \emph{quantitative} versions of the $\mathtt{P}$ operator. For example $\mathtt{P}_{=?}[\psi]$ asks: ``what is the probability of $\psi$ holding?''. In the same way we can add the operators $\mathtt{P_{min=?}}[\psi]$ and $\mathtt{P_{max=?}}[\psi]$ for \gls{mdp} models that ask: ``what is the minimum/maximum probability of $\psi$ holding?''.

\gls{pctl} formulas can be extended with \emph{reward} properties \cite{kwiatkowska2007} by the addition of the \emph{reward operator} $\mathtt{R}_{\bowtie r}[\cdot]$ and the following state formulas:
\begin{equation}
\label{eq:rewards}
	\mathtt{R}_{\bowtie r}[\mathtt{C}^{\leq k}]\; |\; \mathtt{R}_{\bowtie r}[\mathtt{F}\: \phi]
\end{equation}
where $\mathtt{C}$ is the cumulative reward operator, $r \in \mathbb{R}$, $k\in \mathbb{N}$ and $\phi$ is a \gls{pctl} state formula.


\section{The Limited Instruction Set Agent}
\label{sec:rescycle}
The architecture of \gls{lisa}, depicted in Fig. \ref{fig:architecture}, is based on the three-layer architecture \cite{gat1998}. Each block with rounded corners is a collection of so called \emph{skills} that the agent reasoning is able to execute when invoking actions. Note the hybrid nature of the system: the dotted lines represent \emph{symbolic} flows of information, while the solid line represent \emph{numeric} information.

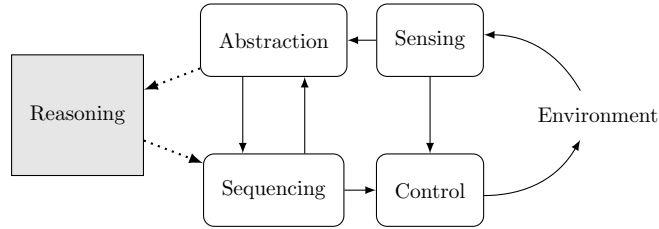
\begin{figure}[htbp]
\centering

	\begin{tikzpicture}[auto,node distance=6mm,>=latex,font=\small,scale=0.8, every node/.style={transform shape}]

\tikzstyle{rect}=[draw=black,inner sep=3mm,fill=white!90!black,text centered]
\tikzstyle{diamond}=[draw=black,minimum height=12mm,minimum width=12mm,shape=diamond]
\tikzstyle{rc}=[draw=black,rounded corners,inner sep=3mm,text centered,minimum height=12mm]
\tikzstyle{round}=[draw=black,inner sep=0mm,text width=10mm, text centered,shape=circle]

\node (reasoning) [rect,minimum height=20mm] {Reasoning};
\node (ref) [right=20mm of reasoning] {};
\node (abstraction) [rc,above=5mm of ref] {Abstraction};
\node (sequencing) [rc,below=5mm of ref] {Sequencing};
\node (ref) [right=23mm of ref] {};
\node (sensing) at (abstraction -| ref) [rc] {Sensing};
\node (control) at (sequencing -| ref) [rc] {Control};
\node (env) [right=15mm of ref,minimum height=8mm] {Environment};

\draw[->,thick,dotted] (reasoning) -- (sequencing);
\draw[->] (sequencing.50) -- (sequencing.50 |- abstraction.south);
\draw[<-] (sequencing.130) -- (sequencing.130 |- abstraction.south);
\draw[->] (sequencing) -- (control);
\draw[->] (control) to [bend right] (env);
\draw[->] (env) to [bend right] (sensing);
\draw[->] (sensing) -- (control);
\draw[->] (sensing) -- (abstraction);
\draw[->,thick,dotted] (abstraction) -- (reasoning);

\end{tikzpicture}
	\caption{The \gls{lisa} architecture}
	\label{fig:architecture}
\end{figure}

The agent program is an evolution of Jason \cite{bordini2007,jasonmanual}. Here follows a brief overview of the modification that were made to Jason.

\emph{Perception}. In \gls{lisa} perception predicates can be of two types: \emph{sensory perception} ($p\in B_s$) and \emph{action feedbacks} ($p\in B_a$), therefore the Beliefs set is defined as: 
	\begin{equation}
		B=\left\{ B_s,B_a,B_m \right\}
	\end{equation}
	where $B_m$ is the set of all possible mental notes. The \emph{action feedbacks} are percepts that actions feedback to the Beliefs set of the agent in order to make the agent aware of the outcome of the action itself, i.e. success, partial success or failure. 
For the purpose of modelling, this classification is very important: the different nature of sensory percepts and action feedbacks needs to be modelled in a different way to accurately describe the behaviour of the environment. Messages are also handled as percepts.

\emph{Goals}. In Jason there is a distinction between beliefs and goals. In a practical sense this distinction does not have a great influence: beliefs and goals can both trigger plans. For this reason in \gls{lisa} we drop the definition of goals, by also defining goals as beliefs. This can simplify the process of generating a model directly from the agent code, by simplifying the syntax, and it also simplifies the modelling of the belief update process, by reducing the number of states required to describe it.

\emph{Logic rules}. In Jason logic-based implication rules are present but yet not well implemented, to the point that the main text itself \cite{bordini2007} advises against their use. In \gls{lisa} we allow for rule to change the Beliefs set and therefore generate events. This feature potentially reduces the state space by allowing the definition of shorter plans with less actions.

In Fig. \ref{fig:lisacycle} we describe the reasoning cycle for \gls{lisa}. 
The first step is to update the Current Beliefs set with the Beliefs Update Function ($f_{BU}$), based on percepts, messages and mental notes, where logic rules are also applied. The Belief Review Function ($f_{BR}$) then checks what changes have been made to the Current Beliefs set and it generates the new Events set. The function $f_P$ gathers all the plans from the Plan Library that are triggered by the current events, if the plan context is applicable to the Current Beliefs set, the plan is copied to the Desires set. An external function called \emph{Plan Selection Function} ($F_O$) selects one plan for each event and it copies it from the Desires set to the Intentions set. Finally for every cycle the function $f_{act}$ executes the next action for each plan. 

\begin{figure}[htbp]
	\centering

\begin{tikzpicture}[auto,node distance=6mm,>=latex,font=\small,scale=0.8, every node/.style={transform shape}]

\tikzstyle{rect}=[draw=black,inner sep=3mm,fill=white!90!black,text width=10mm,text centered]
\tikzstyle{rect2}=[draw=black,inner sep=3mm,text width=10mm,text centered]
\tikzstyle{diamond}=[draw=black,minimum height=13mm,minimum width=13mm,shape=diamond]
\tikzstyle{rc}=[draw=black,rounded corners,inner sep=3.5mm,text centered]
\tikzstyle{round}=[draw=black,inner sep=0mm,text width=10mm, text centered,shape=circle]

	\node (msg) at (0,0) [] {Messages};
	\node (percepts) [above=3mm of msg] {Percepts};
	\node (actfeed) [below=3mm of msg,text width=14mm,text centered,inner sep=0] {Action Feedbacks};
	\node (fbu) [rc,right=10mm of msg] {$f_{BU}$};
	\node (rules) [rect2,above=8mm of fbu] {Logic rules};
	\node (fbr) [rc,right=10mm of fbu] {$f_{BR}$};
	\node (bb) at (rules -| fbr) [rect] {Current Beliefs};
	\node (fp) [rc,right=10mm of fbr] {$f_P$};
	\node (events) at (bb.east -| fp)[rect] {Events};
	\node (des) [rect,below left=6mm and 0 of fp] {Desires};
	\node (Fo) [diamond,right= of des] {$F_O$};
	\node (plans) [rect2,right=5mm of events] {Plan Library};
	\node (inten) [rect,below=of des,minimum height=20mm,minimum width=40mm,text width=20mm] {};
	\node (inten2) [above left=-5mm and -18mm of inten]{Intentions};
	\node (plan1) [fill=white,draw=black,minimum height=12mm,minimum width=8mm,below left=-15mm and -13mm of inten]{$\pi_1$};
	\node (plan2) [fill=white,draw=black,minimum height=10mm,minimum width=8mm,above right=-10mm and 1mm of plan1]{$\pi_2$};
	\node (plan3) [fill=white,draw=black,minimum height=12mm,minimum width=8mm,above right=-12mm and 1mm of plan2]{$\pi_3$};
	\node (act) [minimum height=10mm,rc,right= of inten] {$f_{act}$};
	
	\node () [below left=0 and 0 of fbu.north east] {\scriptsize 1};
	\node () [below left=0 and 0 of fbr.north east] {\scriptsize 2};
	\node () [below left=0 and 0 of fp.north east] {\scriptsize 3};
	\node () [below=0 of Fo.north] {\scriptsize 4};
	\node () [below left=0 and 0 of act.north east] {\scriptsize 5};
	
	\draw[->] (percepts) -- (fbu);
	\draw[->] (msg) -- (fbu);
	\draw[->] (actfeed) -- (fbu);
	\draw[->] (rules) -- (fbu);
	\draw[<->] (bb) -- (fbu.30);
	\draw[->] (bb) -- (fbr);
	\draw[->,name path=line1] (fbr.30) -- (events);
	\draw[->] (events) -- (fp);
	
	\path[name path=line2] (bb) -- (fp.145);
	\path[name intersections={of= line1 and line2}];
	\coordinate(ref) at (intersection-1);
	\path[name path=circ](ref) circle(2pt);
	\path[name intersections={of=circ and line2}];
	\coordinate(ref1) at (intersection-1);
	\coordinate(ref2) at (intersection-2);
	\draw[] (bb) -- (ref1);
	\draw[->] (ref2) -- (fp.145);
	
	\draw[->] (fp) -- (des);
	\draw[->] (plans) -- (fp.30);
	\draw[->] (des) -- (Fo);
	\draw[->] (Fo) -- (plan3.80);
	\draw[->] (Fo) -- (plan2.80);
	\draw[->] (inten) -- (act); 
	\draw[->] (act.15) -- +(5mm,0);
	\draw[->] (act.east) -- +(5mm,0);
	\draw[->] (act.-15) -- +(5mm,0);
	\draw[->] (act) -- +(0,-15mm) -| (inten.south);
	\draw[->] (act) -- +(0,-15mm) -| (fbu);

\end{tikzpicture}
	\caption{The \gls{lisa} reasoning cycle, rounded blocks represent internal functions, white square blocks are static sets, grey blocks are dynamic sets}
	\label{fig:lisacycle}
\end{figure}
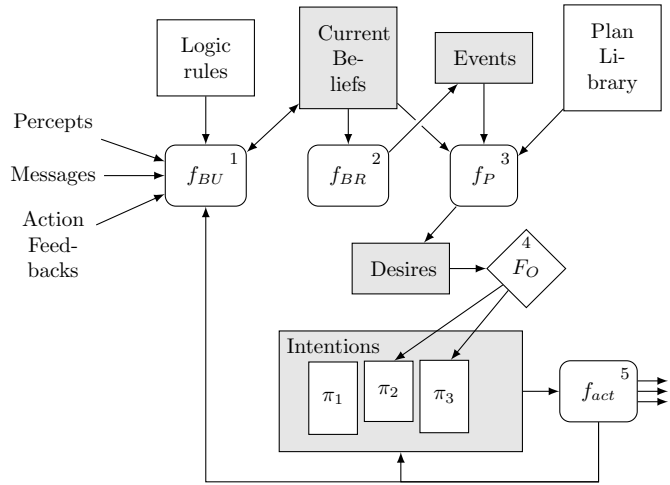

The general flow is similar to that of Jason with mainly one distinction: in every reasoning cycle the Jason agent only allows for the handling of a single event (selected with a function called \emph{Event Selection Function} $F_E$), and then the execution of a single action from the Intentions set (selected with a function called \emph{Intention Selection Function} $F_I$). In \gls{lisa} we implement a multi-threaded work flow that allows the handling of multiple events, and then the execution of multiple actions at the same time. This implies that the \emph{Desires} set becomes:
\begin{equation}
	D[t]=\left\{ D_1[t], \ldots, D_{n_e}[t] \right\}
\end{equation}
where each $D_j[t]$ is the set of plans triggered by an event $e_j\in E[t]$ and $n_e=|E[t]|$ is the number of events. Consequently, the function $F_O$, must be applied to every $D_j[t] \subset D[t]$. It is important to note that plans are \emph{copied} into the Desires set from the Plan library, but not exclusively, which implies that different subsets of $D[t]$ may have a copy of the same plan. However if a plan is selected multiple times in the same reasoning cycle, it will only be executed once. Furthermore once a plan is selected from the Desires set and copied to the Intentions set for execution, if the plan is selected again in the future it will not be executed a second time, but it will carry on from the current state unless a plan interruption action is issued.

This multi-threaded implementation greatly simplifies the modelling process of the agent reasoning by drastically reducing the number of states required to describe it. By eliminating the need for specialised non-deterministic functions the model does not have to keep track of the events and actions activated in previous reasoning cycles therefore reducing the number of states. This also reduces the level of non-determinism in the agent reasoning, which then allows for a more precise generalisation of the abstraction process and in turn the application of an automatic modelling software that generates a complete and verifiable model directly from the agent code.



\section{Abstraction to discrete finite-state machine}
\label{sec:abstraction}

In this section we give a detailed description of the abstraction of the LISA reasoning to two kinds of discrete state machines: \gls{dtmc} and \gls{mdp} (see Definitions \ref{def:dtmc} and \ref{def:mdp}).

The agent defined in Definition \ref{def:agent} is in principle a deterministic system with well defined rules and states. In the Jason implementation however, there are three functions that introduce non-determinism in the reasoning cycle ($F_E$, $F_O$ and $F_I$), which we reduce to one ($F_O$) with our \gls{lisa} implementation. In Theorem \ref{th:dtmc} we show that the \gls{lisa} system can still be modelled as a \gls{dtmc} under the right conditions, and in Theorem \ref{th:mdp} we show that the \gls{lisa} system can always be modelled as a \gls{mdp}.

In Definition \ref{def:agent} we introduced the concept of \emph{plan} as a sequence $\pi = \left\{ \pi(0), \pi(1), \ldots, \pi(n_\lambda) \right\}$. Assuming that a plan is not allowed to be executed multiple times in parallel, let us define a set of \emph{plan indices} $\vec{\lambda}[t] = \left\{ \lambda_1, \lambda_2, \ldots, \lambda_{n_\pi} \right\}$, which represents the state of all plans at time $t$. Note that, according to this definition, a plan $\pi_j$ is a member of the Intentions set $I[t]$ at time $t$ if and only if $\lambda_j>0$ at time $t$. From $\vec{\lambda}[t]$ we can define a set of all possible indices as $\vec{\Lambda} = \left\{ \vec{\Lambda}_1, \vec{\Lambda}_2, \ldots, \vec{\Lambda}_{n_\pi} \right\}$, where $\vec{\Lambda}_j = \{1,\ldots,n_{\lambda_j}\}$ is the set of natural numbers between $1$ and the total number $n_{\lambda_j}$ of actions for each plan $\pi_j$.

\begin{theorem}[\gls{lisa} abstraction to \gls{dtmc}]
\label{th:dtmc} 
	Assuming the existence of sets of (discrete) probability distributions $\dist(B_s)$ and $\dist(B_a)$, over the set of percepts and the set of action feedbacks, if $\;\forall\: i,j\: \in\: [1,n_\pi],\; \pi_i(0) \neq \pi_j(0)$	the \gls{lisa} can be modelled as a \gls{dtmc} .
\end{theorem}
\begin{proof}
	A \gls{dtmc} is completely characterised given a countable set of states $S$ and a transition function $\vec{P}:\: S\times S \rightarrow\: [0,1]$. According to the definition of \gls{lisa}, for a reasoning cycle to be completed the agent needs to be aware of $E[t]$ in order to recall plans from the plan library, $B[t]$ in order to check the plans context, and the state of the plans in $I[t]$ in order to execute the next actions. The state of a \gls{lisa} is only relevant at the end of a reasoning cycle, therefore a generic state can be expressed as $s[t]=\left\{ B[t],E[t],\vec{\lambda}[t] \right\}$.
	The state space, given by $S= \wp(B) \times \wp(B) \times \vec{\Lambda}$, is therefore finite and countable. The state of the agent is initialised by $s_0=\{B_0,\emptyset,\vec{0}\}$ and by triggering the actions listed in the initial actions set $A_0$.\\
	The transition function describes the way in which the state changes at every step. For each reasoning cycle, events can be generated from change in beliefs, namely \emph{mental notes}, \emph{action feedbacks} and \emph{percepts}. Changes in mental notes are given by internal actions, which are known from the plan indices $\lambda$. Changes in action feedbacks and percept are given by known probability distributions. If $\;\forall\: i,j,\; \pi_i(0) \neq \pi_j(0)$, e.g. if all plans have different triggering conditions, then 
	\begin{equation}
		\forall t\in\mathbb{N}_{\geq 1} \;,\;\; \left| \bigcup_{k=1}^{n_e} D_k[t] \right| =  \left| D[t] \right| \leq \left| E[t] \right| 
	\end{equation}
	each event will trigger at most one plan, therefore $F_O$ becomes a trivial one-to-one mapping, therefore the system does not show any non-deterministic behaviour, hence the \gls{lisa} reasoning can be modelled as a \gls{dtmc}.
\end{proof}

\begin{theorem}[\gls{lisa} abstraction to \gls{mdp}]
\label{th:mdp}
	Assuming the existence of sets of (discrete) probability distributions $\dist(B_s)$ and $\dist(B_a)$, over the set of percepts and the set of action feedbacks, the any \gls{lisa} reasoning can be modelled as a \gls{mdp}.
\end{theorem}
\begin{proof}
	A \gls{mdp} is completely described given a countable set of states $S$ and a transition function $Step: S \times C \rightarrow \dist(S)$, with $C(s^\prime)$ being the set of available choices in any state $s^\prime$. The set of states can be built as shown in Theorem \ref{th:dtmc}. \\
	If $\;\forall\: i,j\: \in\: [1,n_\pi],\; \pi_i(0) \neq \pi_j(0)$, according to Theorem \ref{th:dtmc}, the system does not show any non-determinism. However, if $\exists i,j \in [1,n_\pi]: \pi_i(0) = \pi_j(0)$, then 
	\begin{equation}
		\exists t^\prime\in\mathbb{N}_{\geq 1} :\: \left| \bigcup_{k=1}^{n_e} D_k[t^\prime] \right| > \left| E[t^\prime] \right|
	\end{equation}
	the number of applicable plans is greater than the number of events, therefore for some event $e_k[t^\prime]$ ($k\in[1,n_e]$), the application of the Plan Selection Function ($F_O(D_k[t^\prime])=\pi$) involves a non-deterministic choice that implies different future probabilistic outcomes from action feedbacks, which prevents the modelling with \gls{dtmc}. However this choice represents the only non-deterministic part of the agent, thus $C(s^\prime)=D[t^\prime]$. Once a choice is made by the Plan Selection Function, the transitions can be defined by changes in beliefs, given by internal actions and known probability distributions as shown in Theorem \ref{th:dtmc}, and therefore the \gls{lisa} modelling as a \gls{mdp} is complete.
\end{proof}

Probabilistic models such as \glspl{dtmc} and \glspl{mdp} can be verified by means of probabilistic model checking, by using dedicated software such as \textsc{Prism}. Theorems \ref{th:dtmc} and \ref{th:mdp} therefore imply that \gls{lisa} can be verified, assuming that probability distributions of the percepts and action feedbacks are well defined. Theorems \ref{th:dtmc} and \ref{th:mdp} also imply the availability of two options when designing the agent program: to design an agent with all unique triggering conditions, so to possibly improve model checking speed but requiring more effort from the designer, or design an agent with matching triggering conditions so to simplify the design but requiring more computation for the model checking.


\section{Probabilistic modelling within agent programs}
\label{sec:modelling}

In this section we describe the process of modelling probabilistic behaviour of the environment and the action feedbacks in the agent code. The aim is to use a unified approach that allows to obtain a complete model of the agent and its interactions with the environment from a single document. The reasoning of the agent is implemented in sEnglish \cite{lincoln2013,veres2008}, to which we add a few features that give the programmer the option of defining the probabilistic parts of the system. Along with the probabilistic modelling we also introduce a \emph{reward} structure which allows to define and use the reward properties supported by \textsc{Prism}.

The action feedbacks are modelled within the action definition of sEnglish by defining the following three parameters: a \emph{probability value} $p$, the \emph{average number of reasoning cycles} $\mu$ in which the action feedback is expected to become true, and a \emph{variance} $\sigma$. In this way we can simulate a time-delay-uncertain phenomenon without the need for real time models.
For the percept process we use a similar notation with the possibility of defining probability distributions that are conditional to other beliefs. In particular the user defines: a \emph{list of percepts or mental notes} to which the percept being modelled is conditioned to, probability, average number of reasoning cycles and variance of \emph{activation} and \emph{deactivation}.


The last feature we introduce is the possibility for the programmer to describe \emph{reward} structures, that then allow to use reward properties as described in Equation \ref{eq:rewards}. The reward values can be declared by adding a new `\texttt{\{$\cdots$\}}' structure to any percept declaration within the Percept Process section, or to any action within any of the executable plans. 

By specifying all the necessary information, as described above, the designer is able to implement a complete model that includes a probabilistic description of the environment behaviour, e.g. percepts and action feedbacks. This allows to automatically generate \textsc{Prism} input code for verification (see Sec. \ref{sec:translation}).


\section{Design-time verification}
\label{sec:translation}

The software used to perform the design-time and run-time verifications is \textsc{Prism} \cite{kwiatkowska2011,prismwebsite}.
The modelling approach showed in Sec. \ref{sec:modelling} gives an sEnglish program that provides enough information to generate a complete \textsc{Prism} model for verification. The \textsc{Prism} model is generated here with a dedicated Matlab script. The translator only operates on the agent program itself, and it runs in the order of the tens of milliseconds on the laptop PC we used for the testing. For this reason the performances of the translator itself will be considered to be negligible for the results presented in this paper.

The automatically generated \textsc{Prism} model is structured as follows: a variable is defined for every belief (percept, mental note and action feedback), a variable is also defined for every plan, representing the plan index $\lambda$ which captures the state of the plan at any given time.
By using the synchronisation feature offered by the \textsc{Prism} software, the reasoning cycle is simulated in two steps: a \emph{Beliefs set update}, where variables associated with beliefs are updated, and a \emph{plan index update}, where variables associated to plan indexes are updated according to the beliefs.
With this method we ensure that plans only advance when the appropriate conditions on the Beliefs set are met. Note that there are no variables associated with actions as they are not part of the definition of state of the agent, as shown in Theorem \ref{th:dtmc}. 

Note that by using the approach presented in this paper, during the verification process, the user has access to every single belief and plan. This means that the property specification can touch any part of the system, allowing the user to define arbitrarily complex properties on any aspect of the reasoning process. This can be used to drastically reduce the design errors for autonomous agents.

For example, assume that an agent is implemented to have two opposite actions such as `\texttt{go left}' and `\texttt{go right}'. Assuming that the agent is programmed to have $\pi_2(1)$=`\texttt{go left}' and $\pi_4(2)$=`\texttt{go right}', the property:
\begin{itemize}
\item[]
	$\mathtt{P_{max=?}\; [F\;\; (plan\_2=1\; \&\; plan\_4=2)]}$ 
\end{itemize}\vspace{1mm}
will ask the model checker to generate ``the maximum probability of `\texttt{go left}' and `\texttt{go right}' to be executed at the same time at some point in the future''.


\section{Run-time verification for improved decision-making}
\label{sec:runtime}

In this section we propose two different methods for using a run-time verification process as an internal model for improving the decision-making capabilities for the \gls{lisa} system. 
%
The automatically generated \textsc{Prism}, presented in Sec. \ref{sec:translation}, can also be used for run-time verification. Most of the computational power required to verify such a model is usually spent by the model checker when building the model itself, which does not influence the verification time. In other words, once the model is built, the user can run different verification queries without having to rebuild the model. In many cases, PRISM is able to compute the answers to those queries in a matter of seconds, even for a fairly complex model, therefore this can be a reasonable technique to use in this framework.

The first method is to implement the run-time verification process as a \emph{skill} of the agent, e.g. as a module of the full system.
The \gls{dtmc} or \gls{mdp} model is verified against a set of predefined queries. In particular, in \textsc{Prism}, it is possible to check a query by selecting a starting state with the use of \emph{filters} \cite{prismwebsite}. The run-time verification is then used to generate a set of results that will be interpreted by a `\emph{generate beliefs}' function that will activate or deactivate certain beliefs in the agent Beliefs set.

%
%

The second method consists of implementing a \emph{Plan Selection Function} that makes use of model checking to assess probability of success based on user-defined specifications, and selects the most suitable plan. A clear advantage to this approach is that, since the probabilistic model is generated automatically, the user does not need to implement a specialised function for each agent.
A possible implementation for this is as follows. The function takes as input the Current Beliefs and Desires set. The model generated at design time can be initialised with the current state and then checked against predefined queries. This results in probability values that can be used as indices to select the most likely to succeed plan amongst the ones in the Desires set.

%
%


Note that the two methods described here for run-time model checking, are not mutually exclusive: in case the programmer chooses to implement the \gls{lisa} as a \gls{mdp}, they could both be used at the same time.


\section{A case study}
\label{sec:casestudy}

Consider an \gls{asv} designed for mine detection and disposal. The \gls{asv} is equipped with sensing equipment such as sonars and cameras that allow the detection of unidentified objects in the area of interest. These sensors give the vehicle a cone shaped visibility range. Using its pose in the environment and the information from the sensing equipment, the system is able to assess, on the fly, whether or not there has been any area left unclear. All the data collected is continuously sent back to the control centre. Once the mission is started, lower level tasks, such as for example collision avoidance, are carried out automatically by dedicated subsystems. During the exploration of the area, the system tags mine-like objects and logs their positions and available information for the human operators at the control centre to analyse and deliberate. 
In this scenario, a mission consists of a set of points (in terms of latitude and longitude) that outlines a specific area. An algorithm generates a sequence of waypoints connected by linear tracks, the parallel distance between tracks is calculated by considering the range of the available sensing equipment. We will call the linear tracks, and the area surrounding the tracks, ``blocks''.

In a best case scenario the exploration plan will be carried out as it is defined. However a number of problems can occur. In case that the weather condition becomes too harsh, the agent will wait for instructions from human operators. If the agent realises that there are areas left unexplored in the last block, it will make a non-deterministic decision on whether to immediately go back to re-explore missed spots, or keep going and come back at the end of the mission. 

A fragment of the \gls{lisa} program developed for this example is shown in Fig. \ref{fig:agentexample}.
\begin{figure}[htbp]
\centering
	\begin{lstlisting}[language=senglish,frame=single]
PERCEPTION PROCESS
Monitor the following booleans:
//Percepts
Sea state is too high. {[],[0.5,10,0]}
I am at global waypoint.
Areas left unexplored.
Last waypoint reached. {[I am at global waypoint],[1,1,0]}
...
EXECUTABLE PLANS
...
//Plan 5
If ^[Block explored] while ^[Areas left unexplored] and *@\texttildelow @*^[Sea state is too high] then
	[Activate park mode.] 
	[Generate set of waypoints.] 
	+^[Re_exploring areas]
	[Activate drive mode.]. 
...
//Plan 8
If ^[Sea state is too high] while true then
	[Activate park mode.]
	[Wait for instructions.]
	+^[Waiting for instructions]. 	
	\end{lstlisting}
	\caption{Fragment of the agent program for this case study.}
	\label{fig:agentexample}
\end{figure}
In Fig. \ref{fig:prismexample} is shown a fragment of the \textsc{Prism} program that is automatically generated from the agent code.
\begin{figure}[htbp]
\centering
	\begin{lstlisting}[language=prism,frame=single]
module plan_5
plan_5: [0..4] init 0;
[t] plan_5=0 & !(plan_4=0 & block_explored=1 & (areas_left_unexplored=1 & sea_state_is_too_high=0)) -> (plan_5'=0);
[t] plan_5=0 & (plan_4=0 & block_explored=1 & (areas_left_unexplored=1 & sea_state_is_too_high=0)) -> (plan_5'=1);
//activate_park_mode
[t] plan_5=1 & !(park_mode=1) -> (plan_5'=1);
[t] plan_5=1 & (park_mode=1) -> (plan_5'=2);
//generate_set_of_waypoints
...
module wait_for_instructions
continue: [0..5] init 0;
abort: [0..1] init 0;
//continue[0.6,5,0] abort[0.4,5,0] 
[p] !(plan_8=2) & (continue<=1 & abort<=1) -> (continue'=0) & (abort'=0);
[p] (plan_8=2) & (continue<=1 & abort<=1) -> (continue'=5);
	\end{lstlisting}
	
	\caption{Fragment of the agent program for this case study.}
	\label{fig:prismexample}
\end{figure}
In Table \ref{tab:results} results are reported by running the model in \textsc{Prism}. All the testing was done by using an Apple laptop with a dual-core Intel Core i5-4258U 2.4GHz CPU and 16GB of memory running 64-bit Mac OS X 10.11.3. 
We implemented two different versions of agent programs for this case study: one by defining the agent to be abstracted as a \gls{dtmc}, as per Theorem \ref{th:dtmc}, and one as a \gls{mdp}, as per Theorem \ref{th:mdp}. Both of the models feature $10$ executable plans and $3$ logic rules each, $4$ percepts, $4$ possible actions with a total of $5$ action feedbacks, with additional conditions for the \gls{dtmc} version.

Referring to Table \ref{tab:results}, as expected, the \gls{mdp} case generates a model that is larger than the \gls{dtmc} counterpart. However even though the number of states in the \gls{mdp} case is about $70\%$ larger than the number of state for the \gls{dtmc} model, the time required to build the model is very similar. This is possibly due to the way \textsc{Prism} handles the model building: the software constructs a MTBDD \cite{fujita1997} structure, that is very much dependent on the logic structure of the model.

Both models were then ran with a standard verification query that calculates the minimum probability of completing the mission within $100$ steps. Note that for the \gls{dtmc} model the probability is a single value, which is still indicated here with `minimum probability'. The verification time in this case is consistent with the increase in number of states from the \gls{dtmc} model to the \gls{mdp} model.

\begin{table*}
\centering
\renewcommand{\arraystretch}{1.3}
\caption{Verification results for the example.\\ Both models are run with the following query: $\mathtt{P_{min=?}\; [F\leq100\;\; mission\_complete=1]}$.}

\begin{tabular}{|l|r|r|r|r|r|r|r|}	\hline
\textbf{Model} & \textbf{States} & \textbf{Transitions} & \textbf{Choices} & \textbf{Build time} & \textbf{Ver. time} & \textbf{Memory} & \textbf{Result} \\ \hline
\gls{mdp} & $270,268$ & $420,431$ & $276,454$ & $38.061$s & $1.901$s & $10.0$ MB & $0.6357$\\ \hline
\gls{dtmc} & $157,072$ & $231,148$ & N/A & $38.146$s & $2.052$s & $3.5$ MB & $0.6389$ \\
\hline
\end{tabular}

\label{tab:results}
\end{table*}


\section{Conclusions}
In this article we introduced a new architecture for \gls{bdi} agent programming called \acrfull{lisa}. The architecture builds on previous implementations of AgentSpeak but with a simpler structure, in order to facilitate automatic verification of agent reasoning. The reasoning of the agent is defined in the \gls{nlp} language \emph{sEnglish}. We proved that the \gls{lisa} architecture can be abstracted as a Markovian model (\gls{dtmc} or \gls{mdp}), we then showed how to define a full probabilistic model that includes the agent reasoning and the environment, all defined within the agent reasoning program. From the improved agent program we can then automatically generate a full probabilistic model in \textsc{Prism}'s input language for verification.

The model generated from the agent code is used for both design-time and run-time verification. The design-time verification can serve to improve and validate the agent design. The run-time verification can be used to improve the decision-making capabilities of the agent by implementing model-checking techniques as a means of simulation by the agent in order to predict future events and choose the most suitable strategy.

\section*{Acknowledgment}
This work was supported by the EPSRC project EP/J011894/2.


\end{document}